\documentclass[10pt, conference]{IEEEtran}

\usepackage[tight,footnotesize]{subfigure}

\usepackage{amsmath}
\usepackage{amsthm}
\usepackage{amsfonts}
\usepackage{graphicx}
\usepackage{array}
\usepackage{cite}
\usepackage{comment}
\usepackage{lmodern}
\usepackage[]{algorithm2e}
\usepackage{epstopdf}
\newtheorem{theorem}{Theorem}

\usepackage[T1]{fontenc}

\usepackage{textcomp}

\newcommand{\Loss}{\mathcal{L}}

\newcommand{\fxi}{f(x_i)}
\newcommand{\fxj}{f(x_j)}
\newcommand{\fxk}{f(x_k)}
\newcommand{\Real}{\mathbb{R}}

\begin{document}

\title{Probabilistic Formulations of Regression with Mixed Guidance}
\author{Aubrey Gress, Ian Davidson
University of California, Davis\\
adgress@ucdavis.edu, davidson@cs.ucdavis.edu \\
}

\maketitle

\begin{abstract}
Regression problems assume every instance is annotated (labeled) with a real value, a form of annotation we call \emph{strong guidance}.  In order for these annotations to be accurate, they must be the result of a precise experiment or measurement.  However, in some cases additional \emph{weak guidance} might be given by imprecise measurements, a domain expert or even crowd sourcing.  Current formulations of regression are unable to use both types of guidance. We propose a regression framework that can also incorporate weak guidance based on relative orderings, bounds, neighboring and similarity relations. Consider learning to predict ages from portrait images, these new types of guidance allow weaker forms of guidance such as stating a person is in their 20s or two people are similar in age. These types of annotations can be easier to generate than strong guidance. We introduce a probabilistic formulation for these forms of weak guidance and show that the resulting optimization problems are convex. Our experimental results show the benefits of these formulations on several data sets.
\end{abstract}


\section{Introduction}
Regression methods model continuous values ($f(x)$)  from an instance ($x$).  Examples from the internet include estimating age from portraits of images \cite{eidinger2014age} and from science include estimating development stages of fruit fly embryos \cite{zhang2015deep}. However, producing accurate labelings can be difficult because the response must take a precise, continuous value.  

However, in many settings both exact and \emph{approximate} guidance are  available. Consider our motivating example of predicting age from a persons portrait.  While the age of some individuals may be known exactly, others may be unknown but can be approximated by say a human.  We consider four forms of approximate guidance in the regression setting:

\begin{itemize}
\item
\emph{Relative:} $f(x_i) > f(x_j)$. e.g. person $i$ appears to be older than person $j$.

\item
\emph{Bound:} $f(x_i) \in [c_{i_1}, c_{i_2}]$ where $c_{i_1}, c_{i_2} \in \Real$. e.g. person $i$ is in their 20s.

\item
\emph{Neighbor:} $|f(x_i) - f(x_j)| < |f(x_i) - f(x_k)|$. e.g. person $i$ is closer in age to person $j$ than person $k$.

\item
\emph{Similar:} $|f(x_i) - f(x_j)| \leq s$ where $s \in \Real^+$. e.g. person $i$ is close to person $j$ in age.

\end{itemize}

\begin{figure}[h!]
  \centering
\includegraphics[width=1\columnwidth]{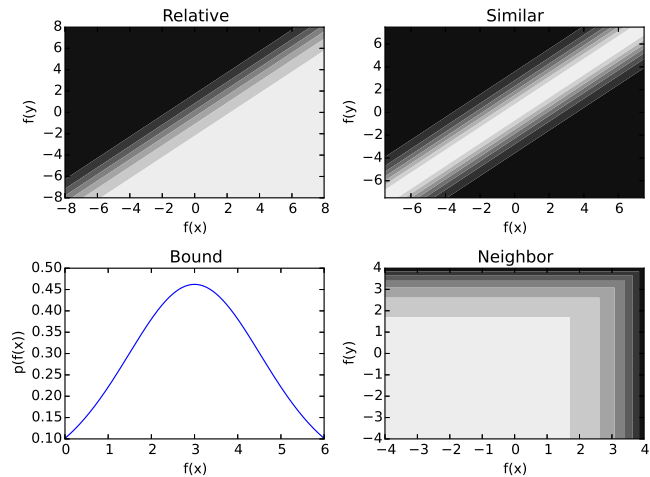}
  \caption{
  Probability distribution functions generated by how we model the different forms of weak guidance.  ``Bound'' is unary so the y axis is the value of the pdf.  The others are binary or ternary, so we use contour plots instead, with lighter regions indicating higher values.  In order to plot Neighbor in two dimensions we fixed the value of one of the instances in the triplet.} 
  \label{fig:distributions}
\end{figure}

We call this \textbf{weak guidance} as opposed to typical annotations which we call \textbf{strong guidance}.  We believe weak guidance can be obtained from a number of sources such as human experts, crowd sourcing or even approximate annotation functions. Consider our example again of estimating ages from portrait images.  If the annotator is sufficiently unsure of a person's age from a photo then they could provide an age interval based on number of years of education.

\textbf{Challenge of Using Weak Guidance.} These new forms of guidance lead to mathematical modeling challenges.  We would like to model them in a manner that captures the generating mechanisms for the guidance, can model the uncertainty in the guidance and can be efficiently optimized.  We solve these problem using additive noise models that give our formulation a probabilistic interpretation and lead to convex optimization problems. Figure \ref{fig:distributions} shows the probability distributions for each of our four annotations.  From Relative we can see that higher probability is associated with regions where $f(x)$ is greater than $f(y)$.  Likewise for Similar when $|f(x)-f(y)|$ is small.  Bound looks similar to a Gaussian distribution but with lighter tails.  Finally, because Neighbor is a ternary constraint, we hard coded $f(x_k) = 4$ and varied $f(x_j)$ on the horizontal axis and $f(x_i)$ on the vertical axis.  We can see that areas where $f(x_j)$ and $f(x_k)$ are far from $4$ have high probability.  



Our contributions are:

\begin{itemize}
\item We introduce a novel way of incorporating both strong and weak guidance into regression (section \ref{sec:regression-with-weak-guidance}).
\item We explore three new forms of weak guidance: Range, Neighbor and Similar (sections \ref{sec:range} through \ref{sec:similar}).  We also develop a new probabilistic formulation for Relative (section \ref{sec:relative}).
\item Our formulation has a probabilistic interpretation, which allows the use of classical statistical methods such as the maximum likelihood estimator.  This also makes it possible for our method to be part of larger probabilistic frameworks (section \ref{sec:regression-with-weak-guidance}).
\item We propose formulations for each form of guidance (sections \ref{sec:relative} through \ref{sec:similar}) and show they are convex (section \ref{sec:convexity}).
\item We experimentally  show our methods outperform semisupervised learning and intuitive baselines (section \ref{sec:experiments}).
\end{itemize}

\begin{table}[t]
\label{tab:notation}
\begin{tabular}{| c | p{6.8cm} |}
\hline
$\epsilon_i$ & A standard logistic random variable. \\ \hline
$F(x)$ & The logistic function, which is the cumulative distribution function of the standard logistic distribution. \\ \hline
$D$ & The set of standard guidance. \\ \hline
$\mathcal{G}$ & The set of weak guidance. \\ \hline
\end{tabular}

\caption{Notation used throughout the paper.}
\end{table}

Our paper is structured as follows.  First, we introduce and propose additive noise models for the four forms of guidance and discuss how they can be used with ridge regression.  We then experimentally compare our method to previous work and baselines.  Finally, we discuss related work and conclude.

Then we discuss and prove the convexity of the resulting optimization problem and discuss optimization issues.

\section{Regression with Weak Guidance}
\label{sec:regression-with-weak-guidance}
We assume a set of weak guidance $\mathcal{G}$ is also available to estimate the regression function $f(x)$ in addition to a small set of labeled training data $D$.  The goal is to use $\mathcal{G}$ to augment the learning process.

We consider four such forms of guidance: \texttt{Relative, Range, Neighbor} and \texttt{Similar}.  We model all forms of guidance probabilistically using additive logistic and exponential noise models.  Our work provides formulations for combining the weak guidance with a more traditional ``base'' estimator that uses the strong guidance.  In all our experiments we used ridge regression as the base estimator.  However, our methods for modeling the weak guidance can be used to augment other methods with probabilistic formulations such as the Lasso, logistic regression and Nadaraya Watson \cite{hastie2005elements}.  

The Maximum Likelihood Estimator (MLE), when using each form of weak guidance with ridge regression, will be of the form:
\begin{align}
\min_w ||Xw - y||^2 + \lambda_1 ||w||^2 + \lambda_2 \sum_{g \in \mathcal{G}} \Loss(g) 
\label{eqn:MLE}
\end{align}
The first two terms are the standard loss and regularization terms that appear in ridge regression.  The last term represents the loss over the weak guidance, where $\Loss$ depends on the form of the weak guidance and $\lambda_2$ is a hyperparameter.

\subsection{Relative Guidance}
\label{sec:relative}
The first form of weak guidance assumes that the  guidance is from a human or imprecise process that can produce estimates of the relative ordering of pairs of responses.  Returning to the problem of regressing age from portrait images, the annotator may not have a good sense of person $i$'s age, but may be confident that person $i$ is younger than person $j$.  We model this as the annotator ordering the instances based on noisy estimates of $\fxi$ and $\fxj$.  Mathematically, we model this as:
\begin{align}
\fxi > \fxj : & \ \ \fxi - \fxj + \epsilon_{ij} > 0 \\
\fxi < \fxj : & \ \ otherwise \nonumber
\label{eqn:relative}
\end{align}
where $\epsilon_{ij}$ is a random variable drawn from a Logistic distribution.  This models the annotator noisily estimating the difference in the pair of responses and producing an ordering based on the sign of the estimate.

It is straightforward to show that $P(\fxi > \fxj) = F(\fxi-\fxj)$ where $F$ is the logistic function $\frac{1}{1+e^{-x}}$.  Using this, relative measurements can be incorporated into a statistical model.  Compared to strong guidance, instead of measuring the loss over instance-label pairs, the loss is over the relative orderings of instance predictions.

Incorporating this guidance into equation \ref{eqn:MLE} leads to the following optimization problem:

\begin{align}
\min_w ||Xw-Y||^2 + \lambda_1 ||w||^2 + \lambda_2(\sum_{(i,j) \in \mathcal{G}}  \log(1 + e^{-(x_i-x_j)^T w}) )
\end{align}
Similar to logistic regression, the Relative guidance term uses the logistic loss.  However, the loss is over \emph{pairs of instances} rather than instance-label pairs.  Intuitively, this term leads to a large penalty when $\fxi$ and $\fxj$ are not in the relative order given by the annotator.

\subsection{Range Guidance}
\label{sec:range}
The next form of guidance assumes the annotators can provide potentially noisy guidance of the form $\fxi \in [a_i,b_i]$ where $a_i$ and $b_i$ are constants.  For the age estimation task, a human may not be confident of an exact age of the person, but may be confident that the age lies in a given range, such as early 20s or mid 30s.

Similar to the previous section, we assume an additive logistic noise model.    We model the probability of the event occurring as $P(f(x) + \epsilon_i \in [a_i, b_i])$.   This becomes:
\begin{align}
& P(f(x) + \epsilon_i \in [a_i,b_i] = \\
& F(b_i-f(x_i)) - F(a_i-f(x_i)) \nonumber
\end{align}

Adding this guidance into equation \ref{eqn:MLE}:
\begin{align}
\label{eqn:bound}
\min_w & \ \ \ ||Xw-Y||^2 + \lambda_1 ||w||^2 -
\\ & \lambda_2(\sum_{i \in \mathcal{G}} \log ((b_i-f(x_i)) - F(a_i-f(x_i))) \nonumber
\end{align}
\subsection{Neighbor Guidance}
\label{sec:neighbor}
Our third form of guidance assumes the annotator can provide guidance of the form $|\fxi - \fxj| < |\fxi - \fxk|$.  For example, given a triplet of images the annotator may be confident that person $i$ is closer in age to person $j$ than $i$ is to $k$.  

Using an additive exponential noise model, we would like to model the probability of this event as $P(|\fxi - \fxj| + \epsilon_{ijk} \leq |\fxk - \fxi|)$ where $\epsilon_{ijk}$ is an exponential random variable, but this leads to a nonconvex optimization problem.  However, by assuming $\fxi < \fxk$, this probability becomes $P(|\fxi - \fxj| + \epsilon_{ijk} \leq \fxk - \fxi)$, which is equivalent to 
\begin{align}
\min ( & 1-H(\fxk - \fxj), \\
& 1-H(\fxk+\fxj - 2 \fxi)) \nonumber
\end{align}
where $H$ is the cumulative distribution function for the exponential distribution.  Assuming that $\fxi$ is close to $\fxj$ relative to $\fxk$, this can be simplified to:

\begin{align}
h(x_i, x_j, x_k) = \min (1-H(\fxk - \fxj), \\
 1-H(\fxk - \fxi)) \nonumber
\end{align}
which we've found works better in practice.

Incorporating this into equation \ref{eqn:MLE} leads to

\begin{align}
\min_w & \ \ \ ||Xw-Y||^2 + \lambda_1 ||w||^2
\\ &  - \lambda_2(\sum_{(i,j,k) \in \mathcal{G}} \log ( h(x_i, x_j, x_k)) \nonumber
\end{align}
which is a convex optimization problem.

We feel this relaxation can be reasonably made.  For example, it may be clear to an annotator that person $i$ is much younger than person $k$, but there may be ambiguity in the relative ordering of $i$ and $j$.  
This form of guidance can be thought of as a combination of Relative over $i$ and $k$, but with additional information about $j$.

\subsection{Similar Guidance}
\label{sec:similar}
The final form of guidance we consider assumes the annotator has a general sense of when $|\fxi - \fxj|$ is relatively small.  For example, an annotator may be confident that two people are roughly the same age.  To model this, we assume that the annotator has some global constant $s$ which they use as ``threshold'' for deciding if two responses are similar.  Modeling this with logistic error becomes $P(|\fxi - \fxj + \epsilon| \leq s)$.  This simplifies to:
\begin{align}
& P(|\fxi - \fxj + \epsilon| \leq s) = \\
& F(s-(\fxi+\fxj)) - F(-s-(\fxi + \fxj)) \nonumber 
\end{align}
For equation \ref{eqn:MLE} the loss will be 
\begin{align}
g(x_i, x_j) = \log ( & F(s-(\fxi+\fxj)) - \\ 
 & F(-s-(\fxi + \fxj)) \nonumber 
\end{align}
and the resulting optimization problem will be:
\begin{align}
\label{eqn:similar}
\min_w & \ \ \ ||Xw-Y||^2 + \lambda_1 ||w||^2
- \lambda_2(\sum_{(i,j) \in \mathcal{G}}  \log (g(x_i, x_j)
\end{align}
As with the previous forms of guidance, this optimization problem is convex.

Our formulation assumes the existence of some constant $s$.  In practice, it could be set by a domain expert.  Alternatively, if it is not known then, as we did in our experiments, it can be tuned like a normal hyperparameter.

\section{Optimization of our Method}
\label{sec:optimization}
Convexity is a desirable property because it means the function can be more efficiently optimized \cite{boyd2004convex}.  These forms of guidance lead to some optimization challenges because while it is trivial to show Relative and Neighbor leads to convex optimization problems, the same is not true for Range and Similar.  Convexity is harder to show for the latter because their formulations lead to terms that have a \emph{difference} of convex functions, which are not convex in general.  However, we will show that the form these terms take lead to convex functions.  Finally, we will discuss numerical issues that may occur when using these methods.  
\subsection{Convexity}
\label{sec:convexity}
The convexity of Relative and Neighbor guidance follow from standard rules of the composition of convex function \cite{boyd2004convex}. Relative is the composition of a negative log, exponential and linear function which all preserve convexity because exponential and negative log are convex and nondecreasing.  Neighbor has a similar form but with an added negation and minimum function, which also preserve convexity.

Proving convexity of Bound and Similar is less straightforward because the log likelihood of these formulations contain terms of the form $\log(F(a(w)) - F(b(w)))$ for some functions $a$ and $b$.  In general, this function need not be convex because it contains a difference of convex functions \cite{boyd2004convex}.  However, we will prove convexity follows from a small set of assumptions which are satisfied by both Similar and Bound guidance.

\begin{theorem}[Convexity of Bound and Similar]
\label{thm:convexity}
Equations \ref{eqn:bound} and \ref{eqn:similar} are convex optimization problems
\end{theorem}

\begin{proof}
Our proof relies on three assumptions:
\begin{enumerate}
\item $a(w) > b(w)$ for feasible $w$
\item $a'(w) = b'(w)$
\item $a''(w) = b''(w) = 0$
\end{enumerate}
where $a'(w)$ denotes the first derivative of $a(w)$.  It is simple to show each of these assumptions hold for Similar and Bound for ridge regression.

Note that assumption (1) is only necessary for $\log(F(a(w)) - F(b(w)))$ to be defined.  To simplify notation we prove it for one dimension, but the multivariate extension follows naturally.  Also, let $a = a(w)$, $b = b(w)$, $F = F(a)$ and $G = F(b)$.

It suffices to show that $f(w) = F - G$ is log concave.  We prove this by showing the following \cite{boyd2004convex}:
\begin{align}
f(w)f''(w) \leq (f'(w))^2
\label{eqn:log-concave}
\end{align}

First the derivatives of $f(w)$:
\begin{align*}
& \frac{d}{dw} F - G = F(1-F)a' - G(1-G)b'\\
& \frac{d^2}{dw^2} F - G = (F-F^2)(1-2F)a'a' - (G-G^2)(1-2G)b'b' \\
\end{align*}
where we've applied assumption (3) to simplify the second derivative.

Using equation \ref{eqn:log-concave}, assumption (2) and much algebra all that remains to be shown is that the following is nonpositive:
\begin{align}
F^4 - F^3 + G^4 - G^3 - 2F^3G - 2FG^3 \\ 
+ 2F^2 G^2 +F^2 G + FG^2 \nonumber
\end{align}

This can be refactored into:
\begin{align}
(F-G)^2(F^2-F+G^2-G)
\end{align}

The leftmost term is nonnegative and because $F,G \in [0,1]$, $F^2-F$ and $G^2-G$ both are nonpositive.  Thus their product is nonpositive, completing the proof.
\end{proof}

\subsection{Numerical Issues}

While all the optimization problems are convex, the likelihood functions of Bound and Similar guidance have a difference of functions, which can lead to numerical stability issues when the difference is close to zero.  To address this we added a small positive constant to each term.  For example, for range guidance the term we used in the log likelihood is $\log(F(b_i-\fxi) - F(a_i-\fxi) + \delta)$ where $\delta$ is a small positive constant.  For our experiments we set $\delta = 1e^{-10}$.

\section{Experiments}
\label{sec:experiments}

Since our work explores new forms of weak guidance for regression our experiments will explore both the usefulness and limitations of weak guidance and also our framework.

Our experiments address the following questions:

\begin{itemize}
\item Does using weak guidance improve performance over using just strong guidance, intuitive baselines and semi-supervised methods? (Figures \ref{fig:synthetic}, \ref{fig:concrete}, \ref{fig:housing}, \ref{fig:fruitfly})
\item Do our methods for implementing weak guidance outperform simple baselines? (Figures \ref{fig:synthetic}, \ref{fig:concrete}, \ref{fig:housing}, \ref{fig:fruitfly})
\item Which of our four types of weak guidance lead to the greatest performance gains in synthetic and real world settings? (Figures \ref{fig:synthetic}, \ref{fig:concrete}, \ref{fig:housing}, \ref{fig:fruitfly})
\end{itemize}

\subsection{Methodology}

We now summarize our experimental design.  First we discuss the baseline methods we used.  Second, how we generated the weak guidance.  Third, how we set hyperparameters.  Finally we discuss the optimization libraries we used.  The data sets we used are summarized in table \ref{tab:data-sets}.

\textbf{Baselines.} Because the forms of guidance we propose are so diverse, we created multiple baselines depending on the form of guidance.  The goal of these baselines is to either compare our work to previously proposed methods or, if no previous work exists, intuitive ways of encoding the guidance into the standard regression setting.  The baselines we used are:

\begin{itemize}
\item Relative guidance $\fxi < \fxj$: Because this type of guidance has been explored before (see section \ref{sec:related-work}), the goal of our experiments is to show our method performs comparably to previous work while having a probabilistic interpretation.  Thus, we use the method proposed by \cite{zhu2007kernel}.
\item Range guidance $\fxi \in [\fxj, \fxk]$: Here, rather than use our weak range guidance we use strong guidance by computing the quartiles of the data set and setting $x_i$'s label to the closest quartile.
\item Neighbor guidance: As mentioned in section \ref{sec:neighbor}, our convex relaxation over a triplet of instances $\{i,j,k\}$ effectively models relative guidance over $i$ and $k$ with the added information that $i$'s response is closer to $j$'s than $k$'s.  As such, we use Relative guidance as a baseline in order to see if this extra information is useful.
\end{itemize}
Because the \texttt{Similar} guidance is so different, we could not think of a reasonable baseline.

In addition to these baselines we also compared all our methods to the following classic and semi-supervised regression methods with the expectation they perform no worse than them. We can view semi-supervised regression methods as a natural competitor to our method as they also use unlabeled instances.
\begin{itemize}
\item Ridge Regression \cite{hastie2005elements}
\item Laplacian Ridge Regression \cite{belkin2006manifold} using a fully connected graph and the Gaussian kernel.
\end{itemize}

For the weak guidance we sampled unlabeled points that were not in the test set and generated the guidance based on their responses as follows:

\begin{itemize}
\item Relative: Uniformally sample pairs of training instances.
\item Range: Uniformally sample instances and bound using upper and lower quartiles.
\item Neighbor: Uniformly sample triplets of training instances.
\item Similar: Set $s=.1\delta$ where $\delta$ is the difference the largest and smallest responses in the data set.  Uniformly sample pairs of instances such that $|y_i - y_j| \leq s$.
\end{itemize}

For all experiments we standardized the covariates and tuned regularization parameters from the set $\{10^k | k \in [-8,-7,...,7,8]\} \cup \{0\}$ (including the bandwidth parameter in Laplacian Ridge Regression) using 10 times random cross validation.  For the Similar guidance we selected $s$ from $\{.05\delta, .1\delta, .2\delta, .3\delta\}$ using 10 times random cross validation.  The data sets we experimented with are summarized in table \ref{tab:data-sets}.

\begin{table}[t]
\label{tab:data-sets}
\begin{tabular}{| c | p{6.1cm} |}
\hline
Synthetic &  We generated a random linear function $w^T x$ where $w$ was sampled from a standard 50 dimensional Gaussian distribution.  We sampled 500 instances $x_i$ from the same distribution, again with each of the 50 components sampled from a standard Gaussian distribution.  We then set $y_i = w^T x_i + \epsilon_i$ where $\epsilon_i$ was drawn from a standard Gaussian distribution. \\ \hline
Concrete \cite{yeh1998modeling} & Predicting the compressive strength of concrete as a function of its age and 7 ingredients. \\ \hline
Housing \cite{harrison1978hedonic} & Predicting housing prices in Boston as a function of various measures of the house and neighborhood. \\ \hline
Fruit Fly \cite{kazmar2013drosophila} & Predicting the developmental stage of fruit fly embryos from images.  We used the already processed features by \cite{kazmar2013drosophila} and selected the 50 covariates most correlated with the response variables. \\ \hline
\end{tabular}
\caption{Data sets used in experiments}
\end{table}

We programmed our solvers using CVXPY \cite{cvxpy} and SciPy's optimize library \cite{scipy}.  The Concrete and Housing data sets were downloaded from the UCI repository \cite{Lichman:2013}.  Code is available at https://github.com/adgress/ICDM2016.

\subsection{Analysis of Experiments}

\texttt{Synthetic Experiments.} Figure \ref{fig:synthetic} shows the results of all four types of guidance in the idealized setting where we know their exists a strong linear relationship between the dependent and independent variables.  We see that for all four types of guidance our method outperforms the baseline of ridge regression.

As to which form of guidance is most useful in this idealized setting, the answer is not conclusive. We find that \texttt{Bound} guidance with small amounts of labeled data provides the largest improvement over the baseline ridge regression method. However, with larger amounts of labeled data Neighbor and Similar guidance perform comparable to \texttt{Bound} guidance. Finally, importantly, we find \texttt{Neighbor} guidance outperforms \texttt{Relative} guidance which indicates that the extra information provided in the former is usefully encoded.

\texttt{Real World Data sets.} 
We next experimented with the concrete data set (Figure \ref{fig:concrete}) which has a less strong linear relationship between the independent and dependent variables. Experimenting with this data set allows us to investigate how our method performed when the assumptions of the ``base estimator,'' ridge regression in this case, are more heavily violated. We find our method performs as well as ridge regression and sometimes better. This makes sense because even if the linearity assumption is violated, extra guidance can aid in estimated the best linear approximation to the true function.  It is notable that Relative performs \textit{worse} than ridge regression when given only 20 pairs, though it performs better with 50 pairs.  We suspect this is due to the increased variance of model selection.  Using weak guidance increases the number of hyperparameters that need to be tuned which can make the method perform worse if the amount of information derived from the weak guidance isn't large enough.  This suggests that some minimum amount of weak guidance is necessary in order for it to be valuable.

The housing data set (Figure \ref{fig:housing}) provides an interesting contrast  to the concrete data set as the linearity assumption is more accurate here. We see that for small amounts of labeled data all forms of weak guidance offer measurable improvements over the baseline methods. For larger amounts of labeled data the guidance does not seem to improve performance.  This suggests that weak guidance can help a system reach the ``peak performance'' faster than using only strong guidance.

Finally, the fruit fly data set (Figure \ref{fig:fruitfly}) gives an important example of how this guidance can be applied to a real world problem.  Due to the poor performance of Laplacian Ridge Regression, we did not include its results in the figure.  We suspect this poor performance is due to the ``cluster assumption.''  As in the previous experiments, Relative, Neighbor and Similar all seem to be useful.

For all but the synthetic data, the performance of Bound doesn't seem to vary dramatically from the baseline we proposed.  While the fact that Bound performs well with synthetic data suggests the method may be useful, the experiments on real data suggest that the simple baseline we proposed can work well enough for many applications.  Future work may be necessary investigate the value of this form of guidance over the baselines.

As we noted before, \cite{zhu2007kernel} suggested a different formulation for Relative, so we compared their formulation to ours. Our implementation of Relative seemed to perform comparably \cite{zhu2007kernel}.  This makes sense because both methods model the same basic form of weak guidance.  However, because \cite{zhu2007kernel} did not model the other forms of weak guidance we proposed, we did not compare their method to our other forms of weak guidance.

\section{Related Work}
\label{sec:related-work}

\textbf{Mixed Guidance in Classification}

Much work has studied the problem of using mixed guidance for classification, such as guidance of the form ``is $f(x_i)$'s label $c$?'' \cite{qi2008two}.

\cite{sculley2010combined} combines regression and ranking losses but the ranking guidance is derived from the given response variables, while our guidance is provided in addition to the given labeled data.

\textbf{Mixed Guidance Regression}

While there has been much work on mixed guidance for classification (such as \cite{qi2008two} and \cite{sculley2010combined}), we have found little work for regression.  This is noteworthy because it is arguable more challenging for a human to provide high quality outputs in the regression setting than the classification setting. 

The work most similar to ours explores guidance of the form ``is $f(x_i) > f(x_j)$''  \cite{zhu2006semi, zhu2007kernel, sculley2010combined}.  They implement this guidance using the hinge loss on pairs of predictions.  These works showed this guidance can be valuable.  However, these works are limited because:

\begin{itemize}
\item Their methods do not have a probabilistic interpretation, meaning standard statistical methods such as AIC and F Statistics cannot be applied.
\item Furthermore, without a probabilistic interpretation these methods do not lend themselves to principled active learning extensions.
\item They only consider relative guidance, while we explore additional forms of guidance.
\end{itemize}

\textbf{Learning to Rank}

Learning to rank is the problem of estimating a function to rank data from a training set that is labeled in some manner \cite{liu2009learning}.  The labels can take several forms such as relative comparisons denoting which item should have higher rank and list-wise comparisons showing how a set of items should be ordered.

While one of our forms of weak guidance, Relative, is similar to what has been explored in the learning to rank settings, the goal of our work is to estimate a regression function rather than an ordering of data.

\section{Conclusion and Future Work}

We proposed a new probabilistic formulation for four forms of weak guidance.     In addition to proposing a new formulation for Relative, we presented three new forms of weak guidance.  Experimentally we showed these forms of guidance can lead to strong performance, even when linearity assumptions do not hold.

A natural extension of this work is how to intelligently make weak-guidance-based queries.  We want to explore new, principled methods for active learning based on weak guidance.  These methods could leverage the probabilistic interpretations we proposed by using techniques from optimal experiment design.

For Similar we assumed $s$ was a constant, but it would be interesting to explore modeling $s$ as a nonconstant function that can vary throughout the domain.

Finally, we suspect there are many new forms of weak guidance yet to explore and are interested in developing new ways to take advantage of weak guidance.

\section{Acknowledgments}

The authors were supported by ONR grant N000141110108 and an Amazon Web Services grant.


\bibliographystyle{IEEEtran}
\bibliography{IEEEabrv,bib}

\begin{figure}[h!]
  \centering
\includegraphics[width=\columnwidth]{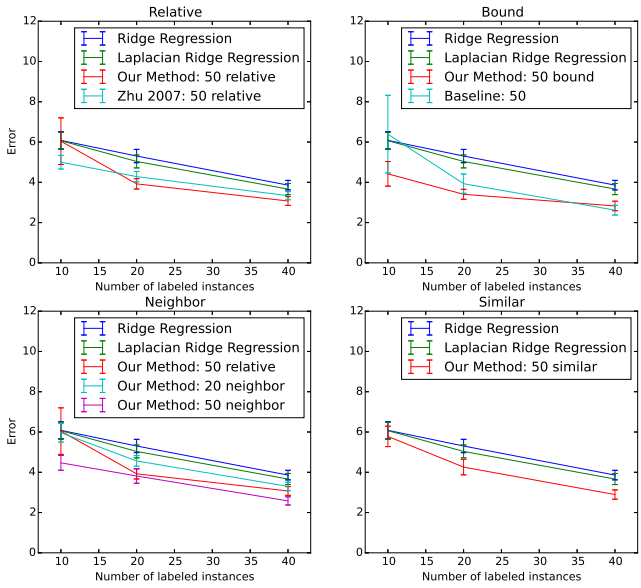}
  \caption{Our guidance on a synthetic data set.} 
  \label{fig:synthetic}
\end{figure}

\begin{figure}[h!]
  \centering
\includegraphics[width=\columnwidth]{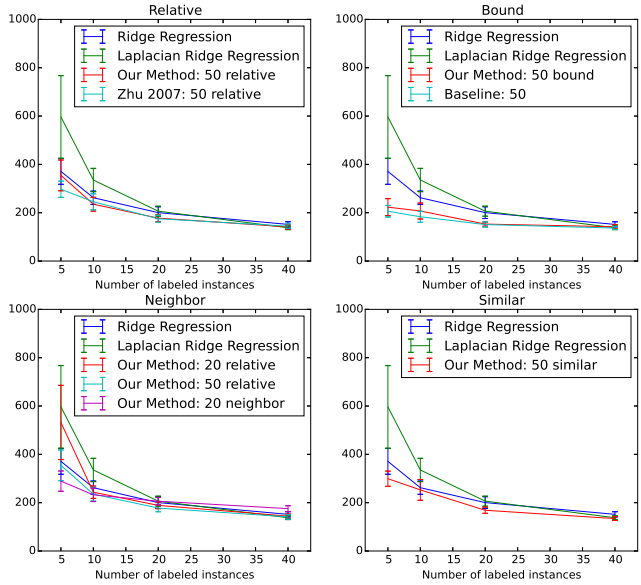}
  \caption{Our guidance on the Concrete Strength data set.}
  \label{fig:concrete}
\end{figure}

\begin{figure}[h!]
  \centering
\includegraphics[width=\columnwidth]{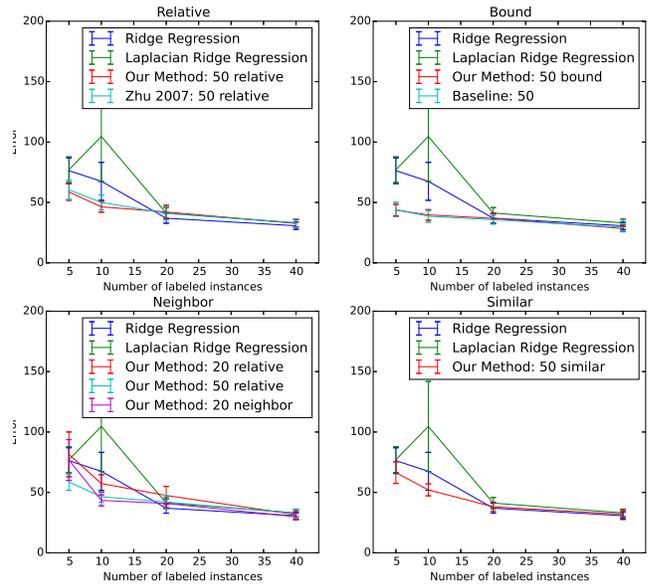}
  \caption{
 Our guidance on the Boston Housing data set.
 }
  \label{fig:housing}
\end{figure}

\begin{figure}[h!]
  \centering
\includegraphics[width=\columnwidth]{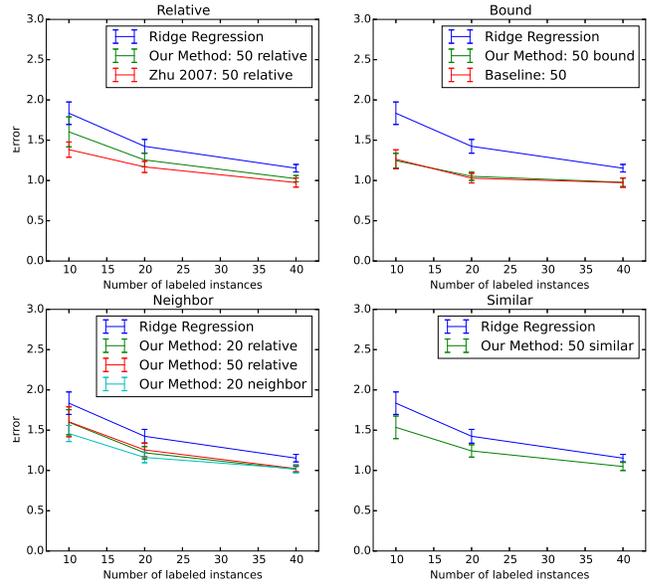}
  \caption{
 Our guidance on the Fruit Fly data set.
 }
  \label{fig:fruitfly}
\end{figure}

\end{document}